\documentclass[letterpaper]{article}
\usepackage{uai2018}
\usepackage[margin=1in]{geometry}

\usepackage{times}

\usepackage{amsmath,amssymb,amsthm}
\usepackage{mathtools}                    
\usepackage{xspace}                       
\usepackage{xparse}                       
\usepackage{bm}
\usepackage{dsfont}                       
\usepackage{caption}
\usepackage{subcaption}
\usepackage{algorithm}
\usepackage[noend]{algorithmic}
\usepackage{adjustbox}
\usepackage{setspace}
\usepackage{natbib}
\usepackage{xcolor}
\usepackage{placeins}
\usepackage{hyperref}
\usepackage{url}

\usepackage{tikz}
\usetikzlibrary{shapes,arrows,fit,backgrounds}
\usepackage{pgfplots}
\pgfplotsset{compat=newest}
\newlength\figureheight
\newlength\figurewidth

\setcitestyle{authoryear,round,citesep={;},aysep={,},yysep={;}}

\definecolor{mydarkblue}{rgb}{0,0.08,0.45}
\hypersetup{
  colorlinks,
  linkcolor={red!50!black},
  citecolor={mydarkblue},
  urlcolor={green!80!black}
}

\usepackage{etoolbox}
\newtoggle{short}
\toggletrue{short}
\togglefalse{short}

\newcommand{\sectref}[1]{\hyperref[#1]{Section \ref*{#1}}}
\newcommand{\chapref}[1]{\hyperref[#1]{Chapter \ref*{#1}}}
\newcommand{\figref}[1]{\hyperref[#1]{Figure \ref*{#1}}}
\newcommand{\figsref}[1]{\hyperref[#1]{Figures \ref*{#1}}}
\newcommand{\subcapref}[1]{\hyperref[#1]{\mbox{(\subref*{#1})}}}
\newcommand{\tabref}[1]{\hyperref[#1]{ Table \ref*{#1}}}
\newcommand{\algoref}[1]{\hyperref[#1]{Algorithm \ref*{#1}}}
\newcommand{\theoremref}[1]{\hyperref[#1]{Theorem \ref*{#1}}}
\newcommand{\lemmaref}[1]{\hyperref[#1]{Lemma \ref*{#1}}}
\newcommand{\lemmasref}[1]{\hyperref[#1]{Lemmas \ref*{#1}}}
\newcommand{\corref}[1]{\hyperref[#1]{Corollary \ref*{#1}}}
\newcommand{\asref}[1]{\hyperref[#1]{Assumption \ref*{#1}}}
\newcommand{\eqtref}[1]{\hyperref[#1]{\mbox{(\ref*{#1})}}}
\newcommand{\appref}[1]{\hyperref[#1]{Appendix \ref*{#1}}}
\newcommand{\propref}[1]{\hyperref[#1]{Proposition \ref*{#1}}}
\newcommand{\lineref}[1]{\hyperref[#1]{line \ref*{#1}}}

\newtheorem{theorem}{Theorem}
\newtheorem{lemma}{Lemma}
\newtheorem{cor}{Corollary}

\newtheorem{prop}{Proposition}

\newcommand{\argmax}{\operatornamewithlimits{argmax}}

\def\*#1{\bm{#1}}

\newcommand{\twopartdefo}[3]
{
	\left\{
		\begin{array}{ll}
			#1\,,& \mbox{if } #2 \\
			#3\,,& \mbox{otherwise}
		\end{array}
	\right.
}

\newcommand{\threepartdefo}[5]
{
	\left\{
		\begin{array}{ll}
			#1\,,& \mbox{if } #2 \\
			#3\,,& \mbox{if } #4 \\
			#5\,,& \mbox{otherwise}
		\end{array}
	\right.
}

\makeatletter
\newcommand{\subalign}[1]{%
  \vcenter{%
    \Let@ \restore@math@cr \default@tag
    \baselineskip\fontdimen10 \scriptfont\tw@
    \advance\baselineskip\fontdimen12 \scriptfont\tw@
    \lineskip\thr@@\fontdimen8 \scriptfont\thr@@
    \lineskiplimit\lineskip
    \ialign{\hfil$\m@th\scriptstyle##$&$\m@th\scriptstyle{}##$\crcr
      #1\crcr
    }%
  }
}
\makeatother

\newcommand*\LET[2]{\STATE #1 $\gets$ #2}



\newcommand{\defeq}{\vcentcolon=}
\renewcommand{\P}{\mathbb{P}}

\newcommand{\tm}{t_{\mathrm{mix}}}
\newcommand{\tme}{\tm(\epsilon)}

\newcommand{\Ms}{$\mathrm{M}^3$}
\newcommand{\bO}{\mathcal{O}}
\newcommand{\Pg}{P^{\mathrm{G}}}
\newcommand{\Pm}{P^{\mathrm{M}}}
\newcommand{\Pc}{P^{\mathrm{C}}}

\newcommand{\bPm}{\bar{P}^{\mathrm{M}}}
\newcommand{\bPc}{\bar{P}^{\mathrm{C}}}
\newcommand{\pmin}{\pi_{\textrm{min}}}
\newcommand{\Fi}{F_i}
\newcommand{\wi}{w_i}
\newcommand{\mi}{m_i}
\newcommand{\ci}{c_i}
\newcommand{\Zi}{Z_i}
\newcommand{\dtv}[2]{d_{\textrm{TV}}\left( #1, #2 \right)}
\newcommand{\pib}{\pi_{\beta}}
\newcommand{\isingb}{$\textrm{\textsc{Ising}}_{\beta}$}
\newcommand{\ising}{\textsc{Ising}}
\newcommand{\dn}{d_n}
\renewcommand{\gg}{\gamma^{\mathrm{G}}}
\newcommand{\gc}{\gamma^{\mathrm{C}}}
\newcommand{\bgm}{\bar{\gamma}^{\mathrm{M}}}
\newcommand{\bgc}{\bar{\gamma}^{\mathrm{C}}}
\newcommand{\Pmax}{P_{\textrm{max}}}

\title{Discrete Sampling using Semigradient-based\\Product Mixtures}

\author{
  {\bf Alkis Gotovos}\\
  ETH Zurich\\
  {\small \texttt{alkisg@inf.ethz.ch}}\\
  \And
  {\bf Hamed Hassani}\\
  University of Pennsylvania\\
  {\small \texttt{hassani@seas.upenn.edu}}\\
  \And
  {\bf Andreas Krause}\\
  ETH Zurich\\
  {\small \texttt{krausea@ethz.ch}}\\
  \And
  {\bf Stefanie Jegelka}\\
  MIT\\
  {\small \texttt{stefje@mit.edu}}
}

\begin{document}

\maketitle

\begin{abstract}
We consider the problem of inference in discrete probabilistic models, that is, distributions over subsets of a finite ground set.
These encompass a range of well-known models in machine learning, such as determinantal point processes and Ising models.
Locally-moving Markov chain Monte Carlo algorithms, such as the Gibbs sampler, are commonly used for inference in such models, but their convergence is, at times, prohibitively slow.
This is often caused by state-space bottlenecks that greatly hinder the movement of such samplers.
We propose a novel sampling strategy that uses a specific mixture of product distributions to propose global moves and, thus, accelerate convergence.
Furthermore, we show how to construct such a mixture using semigradient information.
We illustrate the effectiveness of combining our sampler with existing ones, both theoretically on an example model, as well as practically on three models learned from real-world data sets.
\end{abstract}

\section{INTRODUCTION}
Discrete probabilistic models have played a fundamental role in machine learning.
Examples range from classic graphical models, such as Ising and Potts models \citep{koller09}, which have long been used in computer vision applications \citep{boykov01}, to determinantal point processes \citep{kulesza12} used in video summarization \citep{gong14}, and facility location diversity models used for product recommentation \citep{tschiatschek16}.
Recently, there has been increased interest in general distributions over subsets of a finite ground set $V$; that is, given a set function $F : 2^V \to \mathbb{R}$, distributions of the form $\pi(S) \propto \exp(F(S))$, for all $S \subseteq V$.
These can be equivalently seen as distributions over binary random vectors, if $S$ is replaced by the indicator function of the corresponding vector.
All the aforementioned examples can be expressed in this form for a suitable choice of $F$.

While exact inference in such models is known to be intractable in general \citep{jerrum93}, there has been recent work on analyzing approximate inference techniques, such as variational methods \citep{djolonga14, djolonga16mixed}, and Markov chain Monte Carlo (MCMC) sampling \citep{gotovos15, rebeschini15}.
The sampling analyses, in particular, focus on the Gibbs sampler, and derive sufficient conditions under which it mixes---converges toward the target distribution---sufficiently fast.

Unfortunately, oftentimes in practice these conditions do not hold and the Gibbs sampler mixes prohibitively slowly.
A fundamental reason for this slow mixing behavior is the existence of bottlenecks in the state space of the Markov chain.
Conceptually, one can think about the state-space graph containing several isolated components that are poorly connected to each other, thus making it hard for the Gibbs sampler to move between them.

In this work, we propose a novel sampling strategy that allows for global moves in the state space, thereby avoiding bottlenecks, and, thus, accelerating mixing.
Our sampler is based on using a proposal distribution that approximates the target $\pi$ by a mixture of product distributions.
We further propose an algorithm for constructing such a mixture using discrete semigradient information of the associated function $F$.
This idea makes a step towards bridging optimization and sampling, a theme that has been successful in continuous spaces.
Our sampler is readily combined with other existing samplers, and we show provable theoretical, as well as empirical examples of speedups.

\paragraph{Contributions.}
The main contributions of this paper are as follows.
\begin{itemize}
\item We propose the \Ms{} sampler, which makes global moves according to a specific mixture of product distributions.
\item We theoretically analyze mixing times on an illustrative family of Ising models, and prove that adding the \Ms{} sampler results in an exponential improvement over the Gibbs sampler.
\item We demonstrate the effectiveness of combining the \Ms{} and Gibbs samplers in practice on three models learned from real-world data.
\end{itemize}

\paragraph{Related work.}
Recent work on analyzing the mixing time of MCMC samplers for discrete probabilistic models includes deriving general conditions on $F$ to achieve fast mixing \citep{gotovos15, rebeschini15, li16}, as well as looking at specific subclasses, such as strongly Rayleigh distributions \citep{li16, anari16}.

There has also been work on mapping discrete inference to continuous domains \citep{zhang12, pakman13, dinh17, nishimura18} to enable the use of well-established continuous samplers, such as Hamiltonian Monte Carlo \citep{neal12, betancourt17}.
It is worth pointing out that, while these methods usually outperform simple Gibbs or Metropolis samplers, they still tend to suffer from considerable slowdowns in multimodal distributions \citep{neal12}.
Our work is orthogonal to these methods, in the sense that our proposed sampler can be combined with any of the existing ones to provide a principled way for performing global moves that can lead to improved mixing.

Both darting Monte Carlo \citep{sminchisescu07,ahn13} and variational MCMC \citep{defreitas01} share the high-level concept of combining two chains, one making global moves between high-probability regions, and another making local moves around those regions.
However, their proposed global samplers for continuous spaces are generally not applicable to the class of discrete distributions we consider.

There are several well-known results on mixing of the Gibbs sampler for the Ising model on different graph structures \citep{jerrum93,berger05,levin08,levin08book}.
Other (non-MCMC) approaches to discrete sampling include Perturb-and-MAP \citep{papandreou11,hazan13}, and random projections \citep{zhu15}.
Semigradients of submodular set functions have recently been exploited for optimization \citep{iyer13, jegelka11} and variational inference \citep{djolonga16}, but, to our knowledge, no prior work  has used them for sampling.

\section{BACKGROUND}
We consider set functions $F : 2^V \to \mathbb{R}$, where $V$ is a finite ground set of size $n$ that can be assumed to be $V = \{1, \ldots, n\}$ without loss of generality.
In this paper, we focus on distributions over $\Omega \defeq 2^V$ of the form
\begin{align} \label{eq:pdef}
  \pi(S) = \frac{1}{Z} \exp\left( F(S) \right),
\end{align}
for all $S \in \Omega$.
The partition function $Z \defeq \sum_{S \in \Omega} \exp(F(S))$ serves as the normalizer of the distribution.
Alternatively, we can describe distributions of the above form via binary vectors $X \in \{0, 1\}^n$.
If we define $V(X) \defeq \{v \in V \mid X_v = 1\}$, then the distribution $p_X(X) \propto \exp(F(V(X)))$ over binary vectors is isomorphic to the distribution \eqref{eq:pdef} over sets.

Perhaps the simplest family of such models are log-modular distributions, which describe a collection of independent binary random variables.
Equivalently, they are distributions of the form \eqref{eq:pdef} where $F$ is a modular function, that is, a function of the form $F(S) = c + \sum_{v \in S}m_v$, where $c, m_v \in \mathbb{R}$, for all $v \in V$.
The partition function of a log-modular distribution can be derived in closed form as $Z_m = \exp(c) \prod_{v \in V} \left( 1 + \exp(m_v) \right)$.
Consequently, the corresponding log-modular distribution is
\begin{align*}
  \pi_m(S) = \frac{\exp\big( \sum_{v \in S} m_v \big)}{\prod_{v \in V} \left( 1 + \exp(m_v) \right)}.
\end{align*}

\paragraph{Inference and sampling.}
Performing exact inference in models of the form \eqref{eq:pdef}, that is, computing conditional probabilities such as $\pi(A \subseteq S \subseteq B \mid C \subseteq S \subseteq D)$, is known to be in general \#P-hard \citep{jerrum93}.
As a result, we have to resort to approximate inference algorithms, such as Markov chain Monte Carlo sampling \citep{levin08book}, which is the primary focus of this paper.
An MCMC algorithm for distribution $\pi$ simulates a Markov chain in state space $\Omega$ in such a way that the sequence of visited states $(X_0, X_1, \ldots) \in \Omega^{\mathbb{N}}$ ultimately converges to $\pi$.

\paragraph{Gibbs sampler.}
One of the most commonly used chains is the (single-site) Gibbs sampler, which adds or removes a single element 
at a time.
It first selects uniformly at random an element $v \in V$; subsequently, it adds or removes $v$ to the current state $X_t$ according to the probability of the resulting state.
We denote by $P : \Omega \times \Omega \to \mathbb{R}$ the transition matrix of a Markov chain, that is, for all $S, R \in \Omega$, $P(S, R) \defeq \P\left[ X_{t+1} = R \mid X_t = S \right]$.
Then, if we define
\begin{align*}
p_{S \rightarrow R} = \displaystyle\frac{\exp(F(R))}{\exp(F(R)) + \exp(F(S))},
\end{align*}
and denote by $S \sim R$ states that differ by exactly one element (i.e., $\big||R| - |S|\big| = 1$),
the transition matrix $\Pg$ of the Gibbs sampler is
\begin{align*}
  \Pg(S, R) = 
  \threepartdefo{\displaystyle\frac{1}{n}p_{S \rightarrow R}}{R \sim S}{1 - \displaystyle\sum_{T \sim S} \displaystyle\frac{1}{n}p_{S \rightarrow T}}{R = S}{0}.
\end{align*}

\paragraph{Mixing.}
The efficiency of a Markov chain in approximating its target distribution depends largely on the speed of convergence of the chain, which is quantified by the chain's mixing time.
Most commonly, distance from stationarity is measured by the maximum total variation distance, over all starting states, between $X_t$ and the target distribution $\pi$, that is, $d(t) \defeq \max_{X_0 \in \Omega} \dtv{P^t(X_0, \cdot)}{\pi}$.
Then, the mixing time denotes the minimum number of iterations required to get $\epsilon$-close to stationarity, $\tme \defeq \min \{ t \mid d(t) \leq \epsilon \}$.

A common way to obtain an upper bound on the mixing time of a chain is by lower bounding its spectral gap, defined as $\gamma \defeq 1 - \lambda_2$, where $\lambda_2$ is the second largest eigenvalue of the transition matrix $P$.
The following well-known theorem connects the spectral gap to mixing time.
\begin{theorem}[cf. Theorems 12.3, 12.4 in \citep{levin08book}] \label{thm:spectral}
  Let $P$ be the transition matrix of a lazy, irreducible, and reversible Markov chain, and let $\gamma$ be its spectral gap, and $\pmin \defeq \min_{S \in \Omega} \pi(S)$. Then,
  \begin{align*}
    \left( \frac{1}{\gamma} - 1 \right)\log\left( \frac{1}{2\epsilon} \right) \leq \tme \leq \frac{1}{\gamma} \log\left( \frac{1}{\epsilon\pmin} \right).
  \end{align*}
\end{theorem}

\section{THE MIXTURE CHAIN}
Despite the simplicity and computational efficiency of the Gibbs sampler, the fact that it is constrained to performing local moves makes it susceptible to state-space bottlenecks, which hinder the movement of the chain around the state space.
Intuitively, the state space may contain several high-probability regions arranged in such a way that moving from one to another using only single-element additions and deletions requires passing through states of very low probability.
As a result, the Gibbs sampler may mix extremely slowly on the whole state space, despite the fact that it can move sufficiently fast within each of the high-probability regions.

To alleviate this shortcoming, it is natural to ask whether it is possible to bypass such bottlenecks by using a chain that performs larger moves.
In this paper, we introduce a novel approach that uses a Metropolis chain based on a specific mixture of log-modular distributions, which we call the \Ms{} chain, to perform global moves in state space.
Concretely, we define a proposal distribution
\begin{align} \label{eq:qprop}
  q(S, R) = q(R) &= \frac{1}{Z_q} \sum_{i = 1}^{r} \exp\left( \Fi(R) \right) \nonumber\\
                 &= \frac{1}{Z_q} \sum_{i = 1}^{r} \wi \exp\left(\mi(R) \right),
\end{align}
where each $\Fi(R) = \ci + \sum_{v \in R}m_{iv}$ is a modular function, while each $\mi(R) = \sum_{v \in R}m_{iv}$ is a normalized modular function ($\mi(\emptyset) = 0$), and $\wi = \exp(\ci) > 0$.
If we denote by $\Zi$ the normalizer of $\mi$, then the normalizer of the mixture can be written in closed form as
\begin{align*}
  Z_q = \sum_{R \in \Omega}q(R) &= \sum_{R \in \Omega}\sum_{i = 1}^{r} \wi \exp\left(\mi(R) \right)\\
                                &= \sum_{i = 1}^{r} \wi \sum_{R \in \Omega} \exp\left(\mi(R) \right)\\
                                &= \sum_{i = 1}^r \wi \Zi.
\end{align*}
We define the \Ms{} chain as a Metropolis chain \citep{levin08book} using $q$ as a proposal distribution; its transition matrix $\Pm : \Omega \times \Omega \to \mathbb{R}$ is given by
\begin{align*}
  \Pm(S, R) = \twopartdefo{q(R) p_a(S, R)}{R \neq S}{1 - \displaystyle\sum_{T \neq S} q(T) p_a(S, T)},
\end{align*}
where
\begin{align*}
  p_{a}(S, R) \defeq \min\left\{1, \displaystyle\frac{\pi(R)q(S)}{\pi(S)q(R)}\right\}.
\end{align*}

Note that, contrary to usual practice, the proposal $q$ only depends on the proposed state, but not on the current state of the chain.
As a result, the chain is not constrained to local moves, but rather can potentially jump to any part of the state space.
In practice, \Ms{} sampling proceeds in two steps: first, a candidate set $R$ is sampled according to $q$; then, the move to $R$ is accepted with probability $p_a$.
Sampling from $q$ can be done in $\bO(n)$ time---first, sample a log-modular component, then sample a set from that component.
Computing $p_a$ requires $\bO(r)$ time for the sum in \eqref{eq:qprop}, and it can be straightforwardly improved by parallelizing this computation.
All in all, the total time for one step of \Ms{} is $\bO(n + r)$.

As is always the case with Metropolis chains, the mixing time of the \Ms{} sampler will depend on how well the proposal $q$ approximates the target distribution $\pi$.
The following observation shows that, in theory, we can approximate any distribution of the form \eqref{eq:pdef} by a mixture of the form \eqref{eq:qprop}.

\begin{prop} \label{prop:decomp}
  For any $\pi$ on $\Omega$ as in \eqref{eq:pdef}, and any $\epsilon > 0$, there are positive constants $\wi = \wi(\epsilon) > 0$, and normalized modular functions $\mi = \mi(\epsilon)$, such that, if we define $q(S) \defeq \sum_{i = 1}^r \wi \exp(\mi(S))$, for all $S \in \Omega$, then $\dtv{\pi}{q} \leq \epsilon$.
\end{prop}
Conceptually, the proof relies on having one log-modular term per set in $\Omega$.\footnote{Detailed proofs of all our results can be found in the appendix.}
Therefore, while the above result shows that mixtures of log-modulars are expressive enough, the constructed mixture of exponential size in $n$ is not useful for practical purposes.
On the other hand, it is not necessary for us to have $q$ be an accurate approximation of $\pi$ everywhere, as long as the corresponding \Ms{} chain is able to bypass state-space bottlenecks.
With this in mind, we suggest combining the \Ms{} and Gibbs chains, so that each of them serve complementary purposes in the final chain; the role of \Ms{} is to make global moves and avoid bottlenecks, while the role of Gibbs is to move fast within well-connected regions of the state space.
To make this happen, we define the transition matrix $\Pc : \Omega \times \Omega \to \mathbb{R}$ of the combined chain as
\begin{align} \label{eq:cdef}
  \Pc(S, R) = \alpha\Pg(S, R) + (1-\alpha)\Pm(S, R),
\end{align}
where $0 < \alpha < 1$.
It is easy to see that $\Pc$ is reversible, and has stationary distribution $\pi$.

We next illustrate how combining the two chains works on a simple example, where a mixture of only a few log-modular distributions can dramatically improve mixing compared to running the vanilla Gibbs chain.
Then 
we propose an algorithm for automatically creating such a mixture.

\subsection{EXAMPLE: ISING MODEL ON THE COMPLETE GRAPH} \label{sect:ising}
We consider the Ising model on a finite complete graph \citep{levin08}, also known as the Curie-Weiss model in statistical physics, which can be written in the form of \eqref{eq:pdef} as follows:
\begin{align*}
  \pib(S) = \frac{1}{Z(\beta)}\exp\left(-\frac{2\beta}{n} |S|(n-|S|)\right). \tag{\isingb}
\end{align*}
In particular, we focus on the case where $\beta = \ln(n)$, that is,
\begin{align*}
  \pi(S) = \frac{1}{Z}\exp\left(-\frac{2\ln(n)}{n} |S|(n-|S|)\right). \tag{\ising}
\end{align*}
In this case, if we define $\dn \defeq 2 \ln(n) / n$, then $F(S) = -\dn |S|(n-|S|)$.

The Gibbs sampler is known to experience poor mixing in this model; the following is an immediate corollary of Theorem 15.3 in \citep{levin08book}.
\begin{cor}[cf. Theorem 15.3 in \citep{levin08book}]
  For $n \geq 3$, the Gibbs sampler on \ising{} has spectral gap $\gg = \bO\left(e^{-cn}\right)$, where $c > 0$ is a constant.
\end{cor}
From \theoremref{thm:spectral} it follows that the mixing time of Gibbs is $\tme = \Omega\left((e^{cn} - 1)\log(1/(2\epsilon)) \right)$.
Yet, it has been shown that the only reason for this is a single bottleneck in the state space \citep{levin08}.
To make this statement more formal, let us define a decomposition of $\Omega$ into two disjoint sets, $\Omega_0 \defeq \{S \in \Omega \mid |S| < n/2\}$, and $\Omega_1 \defeq \{S \in \Omega \mid |S| > n/2\}$ \citep{jerrum04}.
To keep things simple, we will assume for the remainder of this section that $n$ is odd; the analysis when $n$ is even follows from the same arguments with only a minor technical adjustment.
Our goal is to separately examine two characteristics of the sampler: (i) its movement between the two sets $\Omega_0$, $\Omega_1$, and (ii) its movement when restricted to stay within each of these sets.

For analyzing the ``between-sets'' behavior, we define the projection $\bar{\pi} : \{0, 1\} \to \mathbb{R}$ of $\pi$ as
\begin{align*}
  \bar{\pi}(i) \defeq \sum_{S \in \Omega_i} \pi(S),
\end{align*}
and, for any reversible chain $P$, we define its projection chain $\bar{P} : \{0, 1\} \times \{0, 1\} \to \mathbb{R}$ as
\begin{align*}
  \bar{P}(i, j) \defeq \frac{1}{\bar{\pi}(i)} \sum_{\subalign{S \in \Omega_i, R \in \Omega_j}} \pi(S) P(S, R).
\end{align*}
It is easy to see that $\bar{P}$ is also reversible and has stationary distribution $\bar{\pi}$. For analyzing the ``within-set'' behavior, we define the restrictions $\pi_i : \Omega_i \to \mathbb{R}$ of $\pi$ as
\begin{align*}
  \pi_i(S) \defeq \frac{\pi_i(S)}{\bar{\pi}(i)},
\end{align*}
and the two restriction chains $P_i : \Omega_i \times \Omega_i \to \mathbb{R}$ of $P$ as
\begin{align*}
  P_i(S, R) \defeq \twopartdefo{P(S, R)}{S \neq R}{1 - \displaystyle\sum_{\subalign{T \in \Omega_i: T \neq S}}P(S, T)}.
\end{align*}
Again, it is easy to see that each of the $P_i$ is also reversible and has stationary distribution $\pi_i$.

Coming back to the Gibbs sampler, if we could show that it mixes fast within each of $\Omega_0$ and $\Omega_1$, then we could deduce that the only reason for the slow mixing on $\Omega$ is the bottleneck between these two sets.
Indeed, the following corollary of a theorem by \cite{ding09} shows exactly that.
\begin{cor}[cf. Theorem 2 in \citep{ding09}] \label{thm:grest}
  For all $n \geq 3$, the restriction chains of the Gibbs sampler $\Pg_i$, $i = 0, 1$, on \ising{} have spectral gap $\gg_i = \Theta\big(\displaystyle\tfrac{2\ln(n) - 1}{n}\big)$.
\end{cor}

To improve mixing we want to create an \Ms{} chain that is able to bypass the aforementioned bottleneck.
For this purpose, we use a mixture of two log-modular distributions, the first of which puts most of its mass on $\Omega_0$, and the second on $\Omega_1$.
We define the mixture of the form \eqref{eq:qdef} by
\begin{align*}
  m_1(S) &= \sum_{v \in S} -\dn (n-1) = -\dn (n-1) |S|,\\
  m_2(S) &= \sum_{v \in S} \dn (n-1) = \dn (n-1) |S|.
\end{align*}
We also use $w_1 = 1 / Z_1$ and $w_2 = 1 / Z_2$, where $Z_1$ and $Z_2$ are the normalizers of $m_1$ and $m_2$ respectively.
It follows that $Z_q = 1 / 2$, and, furthermore, the mixture $q$ is symmetric, that is, $q(S) = q(V \setminus S)$.
Since the proposal $q$ is symmetric and state independent, we would expect the \Ms{} chain to jump between $\Omega_0$ and $\Omega_1$ without being hindered by the bottleneck described previously.
We verify this intuition by proving the following lemma.
\begin{lemma} \label{lem:mproj}
  For all $n \geq 10$, the projection chain $\bPm$ of the \Ms{} sampler on \ising{} has spectral gap $\bgm = \Omega(1)$.
\end{lemma}

Putting everything together we show the following result about the combined chain $\Pc$.
\begin{theorem}
  For all $n \geq 10$, the combined chain $\Pc$ on \ising{} has spectral gap
  \begin{align*}
    \gc = \Omega\left( \displaystyle\frac{2\ln(n) - 1}{2n} \right).
  \end{align*}
\end{theorem}
The proof consists of two steps.
In the first step we make a comparison argument \citep{diaconis93,levin08book} to show that the spectral gaps of the projection and restriction chains of the combined sampler are smaller by at most a constant factor in $\alpha$ compared to those of Gibbs and \Ms{}.
In particular, we use the \Ms{} bound (\lemmaref{lem:mproj}) for the projection chain, and the Gibbs bound (\theoremref{thm:grest}) for the restriction chains.
The second step, then, combines the projection and restriction bounds to establish a bound on the spectral gap of the combined chain.
To accomplish this we use a result by \cite{jerrum04}, which, roughly speaking, states that the spectral gap of the whole chain cannot be much smaller than the smallest of the projection and restriction spectral gaps.

Finally, using \theoremref{thm:spectral}, and noting that, in this case, $\pmin = \bO(e^{-n})$ (cf. proof of \lemmaref{lem:mproj}), we get a mixing time of $\tme = \bO(n^2 \log(1 / \epsilon))$ for the combined chain.
This shows that the addition of the \Ms{} sampler results in an exponential improvement in mixing time over the Gibbs sampler by itself.

\section{CONSTRUCTING THE MIXTURE}
Having seen the positive effect of the \Ms{} sampler, we now turn to the issue of how to choose the proposal $q$.
While a manual construction like the one we just presented for the Ising model may be feasible in some cases, it is often more practical to have an automated way of obtaining the mixture.

Let us assume, as is usually the case, that we have access to a function oracle for $F$, and we want to create a mixture of size $r$.
Ideally, we would like to construct a proposal $q$ that is as close to $\pi$ as possible, that is, minimize an objective such as the following,
\begin{align*}
  E_1(q) &\defeq \min_q \| \pi - q \|\\
         &= \min_q \left\| \frac{\exp(F(\cdot))}{Z} - \frac{1}{Z_q}\textstyle\sum_{i = 1}^r \wi\exp(\mi(\cdot)) \right\|,
\end{align*}
where $\| \cdot \|$ could be, for example, total variation distance or the maximum norm.
Unfortunately, this problem is hard: both computing the partition function $Z$, and jointly optimizing over all $\wi, \mi$ are infeasible in practice.
To make the problem easier, we could try to get rid of the normalizers and weights $\wi$, and iteratively minimize over each $\mi$ individually:
\begin{align*}
  E_2^{(i)}(m_i) \defeq \min_{m_i} \left\| \exp(F(\cdot)) - \textstyle\sum_{j = 1}^{i-1} \exp(\mi(\cdot)) \right\|,
\end{align*}
for $i \in \{1, \ldots, r\}$.
This problem is still hard, since optimizing $\| \exp(F(\cdot)) \|$ is by itself infeasible in general.

\begin{algorithm}[tb]
  \setstretch{1.2}
  \caption{Iterative semigradient-based mixture construction}
  \label{alg:mixture}
  \small{
    \begin{algorithmic}[1]
      \REQUIRE Set function $F$, mixture size $r$
      \FOR{$i = 1$ \TO $r$}
      \LET{$\sigma$}{\textsc{Greedy}($F$, $\{m_1, \ldots, m_{i-1}\}$)} \label{lin:perm}
      \LET{$m_i$}{\textsc{SemiGradient}($F$, $\sigma$)}
      \ENDFOR
      \RETURN $\{m_1, \ldots, m_r\}$
    \end{algorithmic}
  }
\end{algorithm}

\begin{algorithm}[tb]
  \setstretch{1.2}
  \caption{Greedy difference maximization}
  \label{alg:greedy}
  \small{
    \begin{algorithmic}[1]
      \REQUIRE Set function $F$, modular functions $\{m_1, \ldots, m_{i-1}\}$
      \LET{$D_i(S)$}{$F(S) - \log \sum_{j=1}^{i-1} \exp(m_j(S))$, for all $S \in \Omega$}
      \LET{$\sigma$}{$(1, \ldots, n)$}
      \LET{$A$}{$\emptyset$}
      \FOR{$i = 1$ \TO $n$}
      \LET{$v^*$}{$\argmax_{v \in V} \left( D_i(A \cup \{v\}) - D_i(A) \right)$}
      \LET{$\sigma_i$}{$v^*$}
      \LET{$A$}{$A \cup \{v^*\}$}
      \ENDFOR
      \RETURN $\sigma$
    \end{algorithmic}
  }
\end{algorithm}

To arrive at a practical algorithm, we approximate the above objective using the two-step procedure described in \algoref{alg:mixture}.
In the first step, we generate a permutation $\sigma$ of the ground set $V$ by running the greedy algorithm on function $D_i(S) \defeq F(S) - \log \sum_{j=1}^{i-1} \exp(m_j(S))$, as shown in \algoref{alg:greedy}.
Intuitively, the sets that are formed by elements near the beginning of $\sigma$ are those on which $F$ and the current mixture disagree by the most.
Therefore, in the second step, we would like to add to the mixture a modular function $\mi$ that is a good approximation for $F$ on $\{\sigma_1, \ldots, \sigma_k\}$, for a choice of $1 \leq k \leq n$.
To accomplish this, we propose using discrete semigradients.

Semigradients are modular functions that provide lower (subgradient) or upper (supergradient) approximations of a set function $F$ \citep{fujishige05,iyer13}.
More concretely, given a set $S \in \Omega$, a modular function $m$ is a subgradient of $F$ at $S$, if, for all $R \in \Omega$, $F(R) \geq F(S) + m(R) - m(S)$.
Similarly, $m$ is a supergradient if the inequality is reversed.
Although, in general, a function is not guaranteed to have sub- or supergradients at each $S \in \Omega$, it has been shown that this is true when $F$ is submodular or supermodular \citep{fujishige05, jegelka11, iyer12}.

Submodularity expresses a notion of diminishing returns; that is, adding an element to a larger set provides less benefit than adding that same element to a smaller set.
More formally, $F$ is submodular if, for any $S \subseteq R \subseteq V$, and any $v \in V \setminus R$, it holds that $F(R \cup \{v\}) - F(R) \leq F(S \cup \{v\}) - F(S)$.
Supermodularity is defined in a similar way by reversing the sign of this inequality.
The resulting models of the form \eqref{eq:pdef} are referred to as log-submodular and log-supermodular respectively.
Many commonly used models fall under these categories; Ising and Potts models, including our example in the previous section, are log-supermodular, while determinantal point processes and facility location diversity models are log-submodular.

\begin{algorithm}[tb]
    \setstretch{1.2}
	\caption{Subgradient computation}
	\label{alg:sub}
	\small{
		\begin{algorithmic}[1]
			\REQUIRE Set function $F$, permutation $\sigma$
            \LET{$A$}{$\emptyset$}
            \LET{$f$}{$F(\emptyset)$}
			\FOR{$v = 1$ \TO $n$}
			\LET{$m_v$}{$F(A \cup \{\sigma_v\}) - F(A)$}
            \LET{$A$}{$A \cup \sigma_v$}
			\ENDFOR
            \RETURN $m(S) \defeq \sum_{v \in S} m_v$, for all $S \in \Omega$
		\end{algorithmic}
	}
\end{algorithm}

\begin{algorithm}[tb]
    \setstretch{1.2}
	\caption{Supergradient computation}
	\label{alg:super}
	\small{
		\begin{algorithmic}[1]
			\REQUIRE Set function $F$, permutation $\sigma$
            \LET{$k$}{\textsc{DrawUniform}(1, n)}
			\FOR{$v = 1$ \TO $k$}
			\LET{$m_v$}{$F(V) - F(V \setminus \{v\})$}
            \ENDFOR
			\FOR{$v = k+1$ \TO $n$}
			\LET{$m_v$}{$F(\{v\})$}
			\ENDFOR
            \RETURN $m(S) \defeq \sum_{v \in S} m_v$, for all $S \in \Omega$
		\end{algorithmic}
	}
\end{algorithm}

Coming back to the second step of \algoref{alg:mixture}, to create a subgradient of $F$ given permutation $\sigma$ we just need to define a modular function via marginal gains according to the permutation order \citep{iyer13}, as shown in \algoref{alg:sub}.
Moreover, this is a subgradient of $F$ at $\{\sigma_1, \ldots, \sigma_k\}$, for all $1 \leq k \leq n$.
On the other hand, \algoref{alg:super} creates a supergradient of $F$ at $\{\sigma_1, \ldots, \sigma_k\}$ for a randomly chosen $k$. (This type of supergradient is denoted by $\bar{g}_Y$ by \cite{iyer13}.)
In fact, the modular functions $m_1$, $m_2$ that we used in analyzing the Ising model in the previous section were supergradients of $F$ at sets $S_1 = \emptyset$, and $S_2 = V$ respectively.

In practice, we can use \algoref{alg:mixture} regardless of whether $F$ is sub- or supermodular.
We have, however, noticed that subgradients give better results when $F$ is submodular, and the same goes for supergradients and supermodular functions.

\section{EXPERIMENTS}
\setlength\figureheight{0.32\textwidth}
\setlength\figurewidth{0.38\textwidth}
\newcommand{\subflen}{0.33\textwidth}
\newcommand{\scspacey}{-1.8em}
\newcommand{\scspacex}{0.2em}
\begin{figure*}[t!]
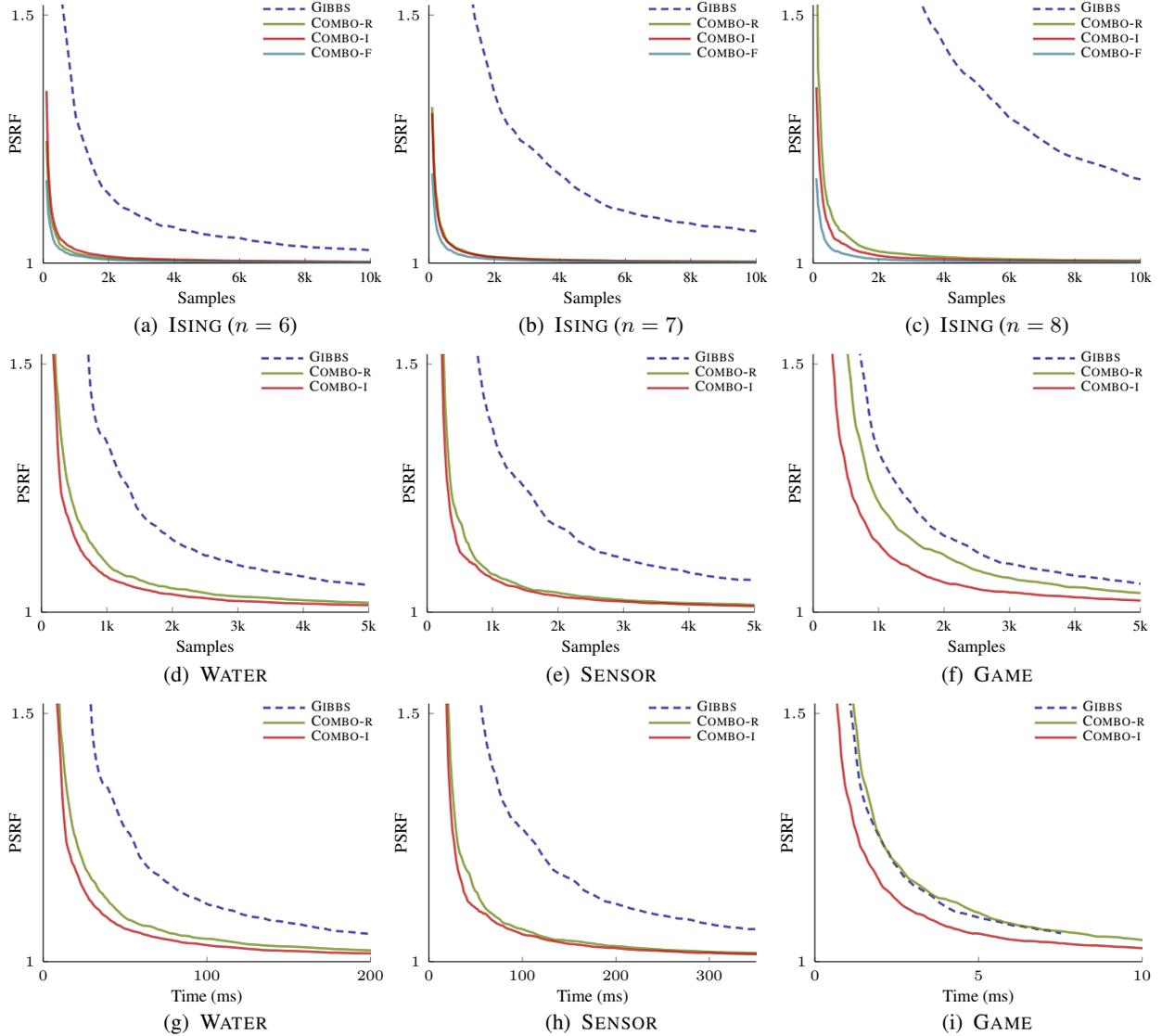

  \captionsetup[subfigure]{oneside,margin={2em,0em}}
  \begin{subfigure}[b]{\subflen}
    \centering
    \input{figures/ising6.tex}
    \vspace{\scspacey}
    \caption{\hspace{\scspacex}\textsc{Ising} ($n = 6$)}
    \label{fig:ising6}
  \end{subfigure}
  \begin{subfigure}[b]{\subflen}
    \input{figures/ising7.tex}
    \vspace{\scspacey}
    \caption{\hspace{\scspacex}\textsc{Ising} ($n = 7$)}
    \label{fig:ising7}
  \end{subfigure}
  \begin{subfigure}[b]{\subflen}
    \input{figures/ising8.tex}
    \vspace{\scspacey}
    \caption{\hspace{\scspacex}\textsc{Ising} ($n = 8$)}
    \label{fig:ising8}
  \end{subfigure}\\
  \begin{subfigure}[b]{\subflen}
    \centering
    \input{figures/water1.tex}
    \vspace{\scspacey}
    \caption{\hspace{\scspacex}\textsc{Water}}
    \label{fig:water1}
  \end{subfigure}
  \begin{subfigure}[b]{\subflen}
    \input{figures/berkeley1.tex}
    \vspace{\scspacey}
    \caption{\hspace{\scspacex}\textsc{Sensor}}
    \label{fig:berkeley1}
  \end{subfigure}
  \begin{subfigure}[b]{\subflen}
    \input{figures/hots1.tex}
    \vspace{\scspacey}
    \caption{\hspace{\scspacex}\textsc{Game}}
    \label{fig:hots1}
  \end{subfigure}\\
  \begin{subfigure}[b]{\subflen}
    \centering
    \input{figures/water1-time.tex}
    \vspace{\scspacey}
    \caption{\hspace{\scspacex}\textsc{Water}}
    \label{fig:water1-time}
  \end{subfigure}
  \begin{subfigure}[b]{\subflen}
    \input{figures/berkeley1-time.tex}
    \vspace{\scspacey}
    \caption{\hspace{\scspacex}\textsc{Sensor}}
    \label{fig:berkeley1-time}
  \end{subfigure}
  \begin{subfigure}[b]{\subflen}
    \input{figures/hots1-time.tex}
    \vspace{\scspacey}
    \caption{\hspace{\scspacex}\textsc{Game}}
    \label{fig:hots1-time}
  \end{subfigure}
  \caption{
    (a)-(c) Ising model results for increasing $n$. Note how the Gibbs sampler gets worse significantly faster than the combined ones.
    (d)-(f) Potential scale reduction factor (PSRF) as a function of sampling iterations.
    (g)-(i) PSRF as a function of wall-clock time in milliseconds.
    The combined sampler outperforms Gibbs both in terms of samples required, as well as actual runtime.
  }
  \label{fig:expising}
\end{figure*}

\setlength\figureheight{0.40\textwidth}
\setlength\figurewidth{0.47\textwidth}
\renewcommand{\subflen}{0.45\textwidth}
\renewcommand{\scspacey}{-0.3em}
\renewcommand{\scspacex}{0.2em}
\begin{figure*}[t!]
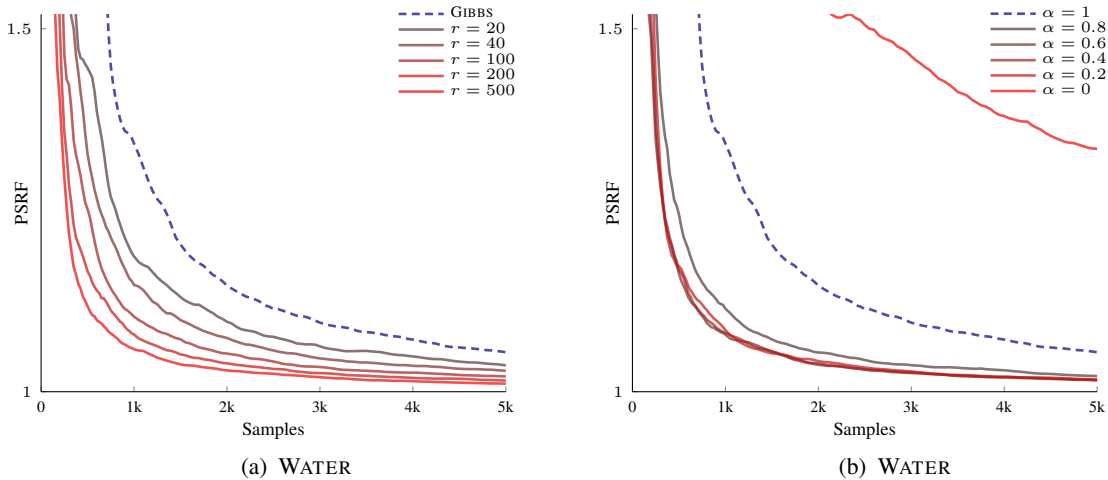

  \captionsetup[subfigure]{oneside,margin={2em,0em}}
  \centering
  \begin{subfigure}[b]{\subflen}
    \centering
    \input{figures/water2.tex}
    \vspace{\scspacey}
    \caption{\hspace{\scspacex}\textsc{Water}}
    \label{fig:water2}
  \end{subfigure}
  \hspace{1em}
  \begin{subfigure}[b]{\subflen}
    \input{figures/water3.tex}
    \vspace{\scspacey}
    \caption{\hspace{\scspacex}\textsc{Water}}
    \label{fig:water3}
  \end{subfigure}
  \caption{
    (a) Increasing the number of mixture components improves performance.
    (b) The combination of Gibbs and \Ms{} performs better than either of them does individually.
  }
  \label{fig:exp2}
\end{figure*}

We now evaluate the performance of our proposed sampler on the Ising model we analyzed earlier, as well as the following three models learned from real-world data sets.
\paragraph{\textsc{Water}.} A (log-submodular) facility location model, which was used in a problem of sensor placement in a water distribution network \citep{krause08}.
The function $F$ is of the form
\begin{align*}
F(S) = \sum_{j = 1}^L \max_{i \in S}c_{ij}.
\end{align*}
We randomly subsample the original facility location matrix $C = (c_{ij})$, so that $n = 50$, and $L = 500$.
\paragraph{\textsc{Sensor}.} A (log-submodular) determinantal point process \citep{kulesza12}, which was used in a problem of sensor placement for indoor temperature monitoring \citep{guestrin05}.
The function $F$ is of the form
\begin{align*}
F(S) = \log |K + \sigma^2 I|,
\end{align*}
where $K$ is a kernel matrix, and $\sigma$ is a noise parameter.
The size of the ground set is $n = 46$.
\paragraph{\textsc{Game}.} A (log-submodular) facility location diversity model \citep{tschiatschek16}, which represents the characters that are chosen by players in the popular online game ``Heroes of the Storm''.
We learned the model from an online data set of approximately $8,000$ teams of $5$ characters\footnote{https://www.hotslogs.com} using noise-contrastive estimation, as described by \cite{tschiatschek16}.
The function $F$ is of the form
\begin{align*}
F(S) = \sum_{v \in S}w_v + \sum_{j = 1}^L \max_{i \in S}c_{ij},
\end{align*}
with $n = 48$, and $L = 10$.
In practice, we would only be interested in sampling sets of fixed size $\ell = 5$.
The Gibbs sampler can be easily modified to sample under a cardinality constraint by using moves that swap an element in the current set $X_t$ with an element in $V \setminus X_t$.
Extending the \Ms{} chain to sample from cardinality-constrained models is also straightforward.
In fact, the only additional ingredient required is a procedure to sample a set of size $\ell$ from a log-modular distribution, which can be easily done, as before, in $\bO(n)$ time.

In what follows, we compare the performance of the Gibbs sampler (\textsc{Gibbs}) against our proposed combined sampler using a proposal mixture $q$ constructed by \algoref{alg:mixture} (\textsc{Combo-I}).
We also compare against a variation where we substitute the greedy procedure in \lineref{lin:perm} of \algoref{alg:mixture} with picking a permutation $\sigma$ of the ground set uniformly at random (\textsc{Combo-R}).

To assess convergence we use the potential scale reduction factor (PSRF) \citep{brooks11} using $20$ parallel chains.
We compute the PSRF using single-element marginal probabilities averaged over $50$ repetitions of each simulation.

In \figsref{fig:ising6}--\ref{fig:ising8} we show the results for the Ising model ($n = 6, 7, 8$) with the additional \textsc{Combo-f} line denoting the combined sampler with two mixture components described in \sectref{sect:ising}.
The other two combined samplers use mixtures of size $r = 20$.
Note that Gibbs mixes dramatically slower than the combined sampler, even for such small $n$.

In \figsref{fig:water1}--\ref{fig:hots1} we show the results on the three log-submodular models described before using mixtures of size $r = 200$.
It is interesting to see that even random permutations are enough to significantly improve over the performance of Gibbs.
Similar observations hold w.r.t. computation time, as shown in \figsref{fig:water1-time}--\ref{fig:hots1-time}, which measure wall-clock time on the $x$-axis.

In \figref{fig:water2} we show how mixture size affects performance; as expected, adding more components to the mixture results in a proposal that approximates the target distribution better, and, therefore, mixes faster.
Finally, in \figref{fig:water3} we see that both Gibbs ($\alpha = 1$) and \Ms{} ($\alpha = 0$, $r = 200$) perform poorly by themselves, but combining them results in much improved performance.
This highlights again the complementary nature of the two chains (local vs. global moves) we discussed earlier.

\section{CONCLUSION}
We considered the problem of sampling from general discrete probabilistic models, and presented the \Ms{} sampler that proposes global moves using a mixture of log-modular distributions.
We theoretically analyzed the effect of combining our sampler with the Gibbs sampler on a class of Ising models, and proved an exponential improvement in mixing time.
We also demonstrated notable improvements when combining the two samplers on three models of practical interest.
We believe that our work represents a step towards moving beyond local samplers, and incorporating ideas from optimization, such as semigradients, into probabilistic inference.

\subsubsection*{Acknowledgements}
This work was partially supported by ERC Starting Grant 307036, NSF CAREER award 1553284, and the Simons Institute for the Theory of Computing.
Any opinions, findings, and conclusions or recommendations expressed in this material are those of the authors and do not necessarily reflect the views of the National Science Foundation.

{\small
\bibliography{uai18}
}
\bibliographystyle{abbrvnat}

\iftoggle{short}
{}
{
  \appendix
  \setcounter{lemma}{0}
  \setcounter{theorem}{0}
  \section{Proof of \propref{prop:decomp}}

\setcounter{lemma}{0}
\begin{prop}
  For any $\pi$ on $\Omega$ as in \eqref{eq:pdef}, and any $\epsilon > 0$, there are positive constants $\wi = \wi(\epsilon) > 0$, and normalized modular functions $\mi = \mi(\epsilon)$, $i \in \{1, \ldots, r\}$, such that, if we define $q(S) \defeq \sum_{i = 1}^r \wi \exp(\mi(S))$, for all $S \in \Omega$, then $\dtv{\pi}{q} \leq \epsilon$.
\end{prop}

\begin{proof}
  Let $r = |\Omega|$, and let $\left( S_i \right)_{i = 1}^r$ be an enumeration of all sets in $\Omega$.
  For any $i \in \{1, \ldots, r\}$, and any $v \in V$, we define
  \begin{align*}
    m_{iv} = \twopartdefo{\beta_i}{v \in S_i}{-\beta_i},
  \end{align*}
  and $\mi(S) = \sum_{v \in S} m_{iv}$, for all $S \in \Omega$.
  We also define
  \begin{align*}
    w_i = \frac{\pi(S_i)}{Z_i} = \frac{\pi(S_i)}{\left(1 + e^{\beta_i}\right)^{|S_i|}\left(1 + e^{-\beta_i}\right)^{|V \setminus S_i|}}.
  \end{align*}
  Then, for all $i \in \{1, \ldots, r\}$, we have
  \begin{align*}
    d_i(&\beta_1, \ldots, \beta_r) \defeq |\pi(S_i) - q(S_i)| \\
      &= \left| \pi(S_i) - \sum_{j = 1}^r \pi(S_j) \frac{e^{\beta_j|S_j|}}{\left( 1 + e^{\beta_j|S_j|} \right) \left( 1 + e^{-\beta_j|V \setminus S_j|} \right)} \right| \\
      &\leq \pi(S_i) \left( 1 - \frac{e^{\beta_i|S_i|}}{\left( 1 + e^{\beta_i|S_i|} \right) \left( 1 + e^{-\beta_i|V \setminus S_i|} \right)} \right) +\\
      &\ \ \ \ \sum_{j : S_j \neq S_i} \pi(S_j) \frac{e^{\beta_j|S_i|}}{\left( 1 + e^{\beta_j|S_j|} \right) \left( 1 + e^{-\beta_j|V \setminus S_j|} \right)}.
  \end{align*}
  Note that both terms vanish if we let all $\beta_j \to \infty$.
  Therefore, for any $\delta > 0$, there are $\beta_{ij} = \beta_{ij}(\delta)$, for all $j \in \{1, \ldots, r\}$, such that $d_i(\beta_{i1}, \ldots, \beta_{ir}) \leq \delta$.
  
  Finally, choosing $\hat{\beta}_j \defeq \max_{i \in \{1, \ldots, r\}} \beta_{ij}$, for all $j \in \{1, \ldots, r\}$, we get
  \begin{align*}
    \dtv{\pi}{q} = \frac{1}{2}\sum_{i = 0}^r d_i(\hat{\beta}_1, \ldots, \hat{\beta}_r) \leq 2^{n-1} \delta.
  \end{align*}
  The result follows by choosing $\delta = \epsilon / 2^{n-1}$.
\end{proof}

\section{Ising Model on the Complete Graph}

\subsection{Bounds on Gibbs mixing}
\let\oldthetheorem\thetheorem
\renewcommand{\thetheorem}{B1}
\begin{theorem}[\hspace{2sp}Theorem 15.3 in \citep{levin08book}] \label{lem:gibbs_exp}
  If $\beta > 1$, then the Gibbs sampler on \isingb{} has a bottleneck ratio $\Phi_{*} = \bO\left(e^{-c(\beta)n}\right)$, where $c(\beta)$ is a non-decreasing function of $\beta$.
\end{theorem}
\let\thetheorem\oldthetheorem

\setcounter{cor}{0}
\begin{cor}[cf. Theorem 15.3 in \citep{levin08book}]
  For $n \geq 3$, the Gibbs sampler on \ising{} has spectral gap $\gg = \bO\left(e^{-cn}\right)$, where $c > 0$ is a constant.
\end{cor}

\begin{cor}[cf. Theorem 2 in \citep{ding09}]
  For all $n \geq 3$, the restriction chains $\Pg_i$, $i = 0, 1$, of the Gibbs sampler on \ising{} have spectral gap $\gg_i = \Theta\left(\displaystyle\frac{2\ln(n) - 1}{n}\right)$.
\end{cor}

\subsection{Bounds on \Ms{} mixing}
\paragraph{\Ms{} sampler.}
The proposal distribution can be written as follows,
\begin{align} \label{eq:qdef}
  q(S) = \frac12\left( \frac{\exp(-\dn (n-1)|S|)}{Z_1} + \frac{\exp(\dn (n-1)|S|)}{Z_2}\right),
\end{align}
where $Z_1 = \left(1 + \exp(-\dn(n-1))\right)^n$, and $Z_2 = \left(1 + \exp(\dn(n-1))\right)^n$.

\let\oldthelemma\thelemma
\renewcommand{\thelemma}{B1}
\begin{lemma} [Fact 6 in \citep{anari16}] \label{lem:anari}
  The spectral gap of any reversible two-state chain $P$ with stationary distribution $\pi$ that satisfies $P(0, 1) = c\,\pi(1)$ is $c$.
\end{lemma}
\let\thelemma\oldthelemma

\setcounter{lemma}{0}
\begin{lemma}
  For all $n \geq 10$, the projection chain $\bPm$ of the \Ms{} sampler on \ising{} has spectral gap $\bgm = \Omega(1)$.
\end{lemma}

\begin{proof}
We define $\pi_k = \sum_{S \in \Omega, |S| = k} \pi(S)$, and $q_k = \sum_{S \in \Omega, |S| = k} q(S)$.

\paragraph{Bounding $\pi_k$.}
By definition, we can write $\pi_k = \hat{\pi}_k / Z$, where $\hat{\pi}_0 = 1$, and for $k > 0$ we have
\begin{align*}
\hat{\pi}_k &\defeq \binom{n}{k} \exp\left(-\frac{2\ln(n)}{n} k(n-k)\right)\\
            &= \binom{n}{k} n^{-\frac{2k}{n} (n-k)}\\
            &\leq \left(\frac{en}{k}\right)^k n^{-\frac{2k}{n} (n-k)}\\
            &= \left(\frac{e}{k}\right)^k n^{-k + \frac{2k^2}{n}}.
\end{align*}
It follows that
\begin{align} \label{eq:logpk}
  \ln(\hat{\pi}_k) \leq -k \ln\left(\frac{k}{e}\right) + \left(\frac{2k^2}{n} - k\right)\ln(n).
\end{align}
It is easy to verify that for all $n \geq 10$ and $3 \leq k \leq \lfloor n/2 \rfloor$, it holds that $(2k-n)\ln(n) \leq 0.5n\ln(k/e)$.
Substituting this into \eqref{eq:logpk}, we get
\begin{align*}
            \ln(\hat{\pi}_k) &\leq -0.5k\ln\left(\frac{k}{e}\right)\\
  \Rightarrow\ \ \hat{\pi}_k &\leq \exp(-0.5k\ln(k/e)).
\end{align*}
Noting that, for all $k$, $\hat{\pi}_k \leq 1$, and using the fact that $\hat{\pi}_{n-k} = \hat{\pi}_k$, we get
\begin{align}
  Z &= \sum_{k = 0}^n \hat{\pi}_k \nonumber\\ 
    &\leq 2\sum_{k = 0}^{\lfloor n/2 \rfloor} \hat{\pi}_ k\nonumber\\
    &= 2(\hat{\pi}_0 + \hat{\pi}_1 + \hat{\pi}_2 + \sum_{k = 3}^{\lfloor n/2 \rfloor} \hat{\pi}_k) \nonumber\\
    &\leq 3 + \sum_{k = 3}^{\lfloor n/2 \rfloor} \exp(-0.5k\ln(k/e)) \nonumber\\
    &\leq c_1, \label{eq:Zconst}
\end{align}
where $c_1$ is a constant.

\paragraph{Bounding $q_k$.}
First, it is easy to see that, for all $n \geq 1$, $Z_1 \leq 3$.
\begin{align*}
  q_k &= \sum_{S \in \Omega, |S| = k} q(S)\\
      &\geq \sum_{S \in \Omega, |S| = k} \frac{1}{2} \frac{\exp(-\dn (n-1)|S|)}{Z_1} \tag{by \eqref{eq:qdef}} \\
      &\geq \frac{1}{6} \binom{n}{k} \exp(-\dn (n-1)|S|)
\end{align*}

\paragraph{Bounding the spectral gap.}
For the projection chain $\bPm$, we have
\begin{align*}
\bPm(0, 1) &= \frac{1}{\bar{\pi}(0)} \sum_{\subalign{S &\in \Omega_i\\ R &\in \Omega_j}} \pi(S) \Pm(S, R)\\
           &\geq 2\pi_0 q_n \tag{$\bar{\pi}(0) = 1/2$ by symmetry of $\pi$}\\
           &= 2\pi_0 q_0 \tag{by symmetry of $q$}\\
           &\geq 2\frac{\hat{\pi}_0}{Z} \frac{1}{6} \tag{$q_0 \geq \frac{1}{6}$}\\
           &\geq 2\frac{1}{c_1}\frac{1}{6} \tag{$\hat{\pi}_0 = 1$}\\
           &= c \bar{\pi}(1),
\end{align*}
where $c = (2/3)c_1$.

Finally, it follows from \lemmaref{lem:anari} that the spectral gap of $\bPm$ is $c$.

\end{proof}

\subsection{Bounds on combined sampler mixing}

\let\oldthelemma\thelemma
\renewcommand{\thelemma}{B2}
\begin{lemma} \label{lem:cproj}
  For all $n \geq 10$, the projection chain $\bPc$ of the combined chain on \ising{} has spectral gap $\bgc = \Omega(1).$
\end{lemma}
\let\thelemma\oldthelemma

\begin{proof}
  By definition, $\bPc(S, R) \geq \alpha \bPm(S, R)$, therefore a simple comparison argument (e.g., Lemma 13.22 in \citep{levin08book}) combined with the result of \lemmaref{lem:mproj} gives us $\bgc \geq \alpha \bgm = \Omega(1)$.
\end{proof}

\let\oldthelemma\thelemma
\renewcommand{\thelemma}{B3}
\begin{lemma} \label{lem:crest}
  For all $n \geq 3$, each of the restriction chains $\Pc_i$ of the combined chain on \ising{} has spectral gap $\gc_i = \Theta\left(\displaystyle\frac{2\ln(n) - 1}{2n}\right)$.
\end{lemma}
\let\thelemma\oldthelemma

\begin{proof}
  By definition, $\Pc_i(S, R) \geq \alpha \Pg_i(S, R)$, therefore, using a comparison argument like above together with \lemmaref{thm:grest} gives us $\gc_i \geq \alpha \gg_i = \Theta\left(\displaystyle\frac{2\ln(n) - 1}{2n}\right)$.
\end{proof}

\let\oldthetheorem\thetheorem
\renewcommand{\thetheorem}{B2}
\begin{theorem}[\hspace{2sp}Theorem 1 in \citep{jerrum04}] \label{thm:jerrum04}
  Given a reversible Markov chain $P$, if the spectral gap of its projection chain $\bar{P}$ is bounded below by $\bar{\gamma}$, and the spectral gaps of its restriction chains $P_i$ are uniformly bounded below by $\gamma_{\textrm{min}}$, then the spectral gap of $P$ is bounded below by
  \begin{align*}
    \gamma = \min\left\{ \frac{\bar{\gamma}}{3}, \frac{\bar{\gamma}\gamma_{\textrm{min}}}{3\Pmax + \bar{\gamma}} \right\},
  \end{align*}
  where $p_{\textrm{max}} \defeq \displaystyle\max_{i \in \{0, 1\}}\max_{S \in \Omega_i} \sum_{R \in \Omega \setminus \Omega_i} P(S, R)$.
\end{theorem}
\let\thetheorem\oldthetheorem

\setcounter{theorem}{1}
\begin{theorem}
  For all $n \geq 10$, the combined chain $\Pc$ on \ising{} has spectral gap
  \begin{align*}
    \gc = \Omega\left( \displaystyle\frac{2\ln(n) - 1}{2n} \right).
  \end{align*}
\end{theorem}

\begin{proof}
  The result follows directly by combining the spectral gap bounds of \lemmasref{lem:cproj} and \ref{lem:crest} in \theoremref{thm:jerrum04}, and noting that $\Pmax \leq 1$.
\end{proof}
}

\end{document}